\documentclass[10pt,twocolumn,letterpaper]{article}

\usepackage{iccv}
\usepackage{times}
\usepackage{epsfig}
\usepackage{graphicx}
\usepackage{amsmath}
\usepackage{amssymb}
\usepackage{pifont}
\usepackage{booktabs}
\usepackage{multirow}
\usepackage{balance}
\usepackage{algorithm}
\usepackage{algorithmic}
\usepackage[utf8]{inputenc}
\usepackage[english]{babel}
\usepackage{amsthm}
\usepackage{color}
\newtheorem{theorem}{Theorem}
\newtheorem{lemma}[theorem]{Lemma}

\newtheorem{definition}[theorem]{Definition}
\newtheorem{corollary}[theorem]{Corollary}
\newcommand{\cmark}{\ding{51}}%
\newcommand{\xmark}{\ding{55}}%

\def\sub#1{_{\rm #1}}
\def\vct#1{\mbox{\boldmath $#1$}}
\def\eg{{\it e.g.}}

\def\ie{{\it i.e.}}

\def\ui#1{^{(#1)}}
\def\perm{\vct{\phi}}
\def\he#1{\vct{\zeta}\left(#1\right)}

\def\hes#1{\vct{\zeta}(#1)}
\def\dperm#1{\vct{\Phi}\left(#1\right)}

\def\nz#1{nz(#1)}
\def\barw{\bar{\vct{w}}}
\def\wui{\vct{w}\ui{n}}

\usepackage[breaklinks=true,bookmarks=false]{hyperref}

\iccvfinalcopy 

\def\iccvPaperID{641} 

\ificcvfinal\fi
\begin{document}

\title{Privacy-Preserving Visual Learning Using \\Doubly Permuted Homomorphic Encryption}

\author{Ryo Yonetani\\
The University of Tokyo\\
Tokyo, Japan\\
{\tt\small yonetani@iis.u-tokyo.ac.jp}
\and
Vishnu Naresh Boddeti\\
Michigan State University\\
East Lansing, MI, USA\\
{\tt\small vishnu@msu.edu}
\and
Kris M. Kitani\\
Carnegie Mellon University\\
Pittsburgh, PA, USA\\
{\tt\small kkitani@cs.cmu.edu}
\and
Yoichi Sato\\
The University of Tokyo\\
Tokyo, Japan\\
{\tt\small ysato@iis.u-tokyo.ac.jp}
}

\maketitle
\thispagestyle{empty}

\begin{abstract}
We propose a privacy-preserving framework for learning visual classifiers by leveraging distributed private image data. This framework is designed to aggregate multiple classifiers updated locally using private data and to ensure that no private information about the data is exposed during and after its learning procedure. We utilize a homomorphic cryptosystem that can aggregate the local classifiers while they are encrypted and thus kept secret. To overcome the high computational cost of homomorphic encryption of high-dimensional classifiers, we (1) impose sparsity constraints on local classifier updates and (2) propose a novel efficient encryption scheme named doubly-permuted homomorphic encryption (DPHE) which is tailored to sparse high-dimensional data. DPHE (i) decomposes sparse data into its constituent non-zero values and their corresponding support indices, (ii) applies homomorphic encryption only to the non-zero values, and (iii) employs double permutations on the support indices to make them secret. Our experimental evaluation on several public datasets shows that the proposed approach achieves comparable performance against state-of-the-art visual recognition methods while preserving privacy and significantly outperforms other privacy-preserving methods.
\end{abstract}

\section{Introduction}
\begin{figure}[t]
\centering
\includegraphics[width=\linewidth]{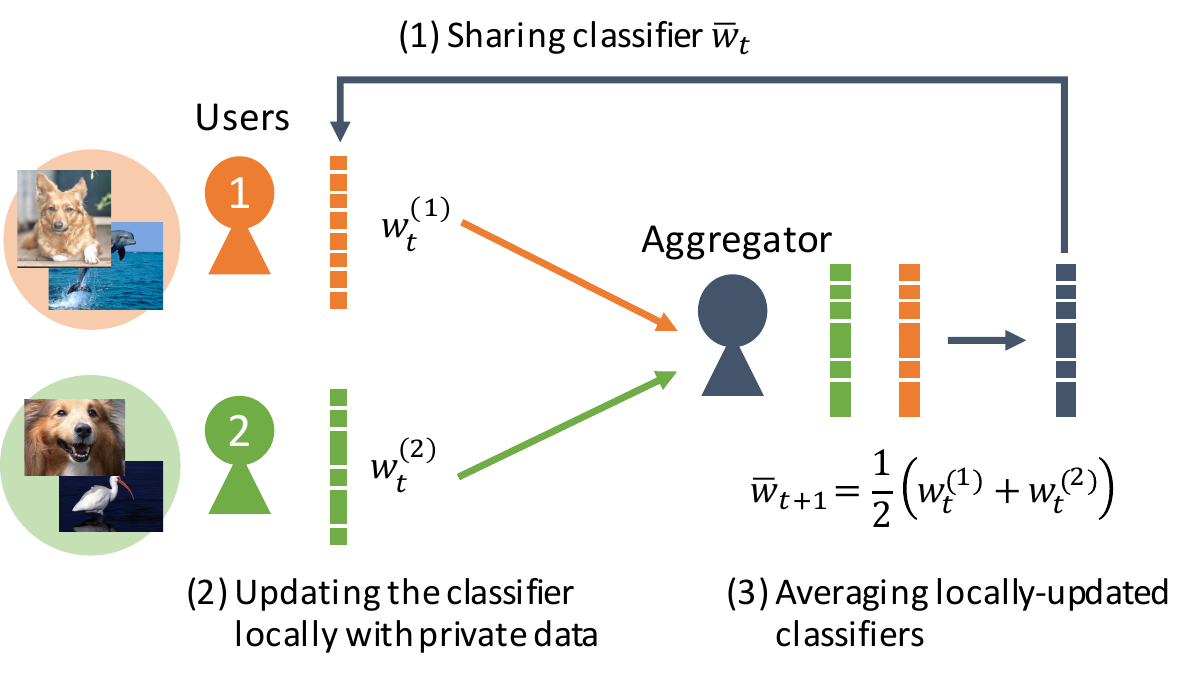}
\caption{{\bf Learning a Classifier by an Aggregator and Users}. (1) Given classifier $\bar{\vct{w}}_t$ shared by an aggregator (\eg, the administrator of a cloud storage); (2) users update it locally using private data and send updated classifier $\vct{w}\ui{1}_t, \vct{w}\ui{2}_t$ to the aggregator; and (3) the aggregator collects and averages the locally-updated classifiers to obtain a better one: $\bar{\vct{w}}_{t+1} = \frac{1}{2}(\vct{w}\ui{1}_t + \vct{w}\ui{2}_t)$. How can users keep $\vct{w}\ui{1}_t, \vct{w}\ui{2}_t$ secret from the aggregator and other users?}
\label{fig:teaser}
\end{figure}

An enormous amount of photos and videos are captured everyday thanks to the proliferation of camera technology. Such data has played a pivotal role in the progress of computer vision algorithms and the development of computer vision applications. For instance, a collection of photos uploaded to Flickr has enabled large-scale 3D reconstruction \cite{Agarwal2011a} and city attribute analysis \cite{Zhou2014a}. More recently, a large-scale YouTube video dataset has been released to drive the development of the next phase of new computer vision algorithms \cite{Sami2016}. Arguably, much of the progress made in computer vision in the last decade has been driven by data that has been shared publicly, typically after careful curation.

Interestingly, we do not always share captured data publicly but just store it on a personal storage~\cite{Hoyle2014}, and such private data contains various moments of everyday life that are not found in publicly shared data~\cite{Ahern2007,Chowdhury2016,Hoyle2015,Kairam2016}. Accessing extensive amounts of private data could lead to tremendous advances in computer vision and AI technology. However there is a dilemma. On one hand, leveraging this diverse and abundant source of data could significantly enhance visual recognition capabilities and enable novel applications. On the other hand, using private data comes with its own challenge: \emph{how can we leverage distributed private data sources while preserving the privacy of the data and its owners?}

To address this challenge, we develop a privacy preserving framework that can learn a visual classifier from distributed private data. As illustrated in Figure~\ref{fig:teaser}, we consider the following two types of parties: users who own private data and an aggregator who wishes to make use of them. They collaboratively learn a classifier as follows: (1) given a classifier shared by the aggregator; (2) users update it locally with their own data; and (3) the aggregator collects and averages the locally-updated classifiers to obtain a better one. These steps iterate multiple times until the aggregator and the users get a high-performance classifier which has been trained on a large collection of private data. As a practical example, suppose that people capture life-logging videos passively with a wearable camera and store them on a cloud storage. By allowing the cloud administrator to curate these videos, he will be able to learn activity classifiers~\cite{Fan2016a,Fathi2012,Lee2015,Pirsiavash2012,Ryoo2013,Singh2016} from realistic and diverse data. However, because such videos inevitably include any private moment that came into the view of cameras, the administrator should ensure that no private information in the videos is leaked during and after the learning procedure.

One possible solution for preserving privacy is to perturb classifier weights by adding noise \cite{Pathak2010a,Rajkumar2012a}. Users could also sanitize their data by detecting and removing sensitive images \cite{Fan2016a,Korayem2016a,Templeman2014a} or transforming images into low-resolution ones \cite{Ryoo2017a}. We stress here that these approaches however involve strong trade-off between data degradation and data utility -- the more information we hide, the less effective the data is for learning classifiers. 

An attractive solution for privacy-preserving learning without such trade-off is the use of \emph{homomorphic cryptosystems} --- a class of cryptosystems that can allow for basic arithmetic operations over encrypted data \cite{Fontaine2007a}. For example, the Paillier cryptosystem~\cite{Paillier1999a} calculates the encrypted sum of data as follows: $\he{x\ui{1} + x\ui{2}}=\he{x\ui{1}}\he{x\ui{2}}$ where $\he{x}$ is the ciphertext of value $x$ and \emph{never reveals the original $x$} without a decryption key. This unique property of homomorphic cryptosystems has enabled a variety of privacy-preserving machine learning and data mining techniques (\eg, \cite{Bost2015a,Dowlin2016a,Graepel2012,Pathak2010a,Takabi2016a,Yuan2014a}). By homomorphically encrypting locally-updated classifiers, the aggregator can average them while ensuring that the classifiers never expose private information about the trained data.

However, homomorphic encryption involves prohibitively high computational cost rendering it unsuitable for many tasks that need to learn high-dimensional weights. For example, in our experiments, we aim to learn a classifier for recognizing 101 objects in the Caltech101 dataset \cite{Fei-Fei2007a} using 2048-dimensional deep features. Paillier encryption with a 1024-bit key takes about 3 ms for a single weight using a modern CPU\footnote{A single CPU of a MacBook Pro with a 2.9GHz Intel Core i7 was used with python-paillier (\url{https://github.com/n1analytics/python-paillier}) and gmpy2 (\url{https://pypi.python.org/pypi/gmpy2}; a C-coded multiple-precision arithmetic module.)}. Therefore, encryption of each classifier will require more than 10 minutes.

To leverage homomorphic encryption for our privacy-preserving framework, we present a novel encryption scheme named \emph{doubly-permuted homomorphic encryption (DPHE)}, which allows high-dimensional classifiers to be updated securely and efficiently. The key observation is that, if we enforce sparsity constraints on the classifier updates, the updated weights can be decomposed into a small number of non-zero values and corresponding support indices. DPHE then applies homomorphic encryption only to non-zero values. Because the support indices could also be private information (\eg, when classifier weights take binary values \cite{Cheng2014a,Courbariaux2015a,Hare2012a}), we design a shuffling algorithm of the indices based on two different permutation matrices. This shuffling ensures that 1) only the data owners can identify the original indices and 2) an aggregator can still average classifiers while they are encrypted and shuffled. By adopting a sparsity constraint of more than 90\%, DPHE reduces the classifier encryption time of the previous example on Caltech101 to about one minute.

We evaluate our visual learning framework on a variety of tasks ranging from object classification on the classic Caltech101/256 datasets~\cite{Fei-Fei2007a,Griffin2007a} to more practical and sensitive tasks with a large-scale dataset including face attribute recognition on the CelebA dataset~\cite{Liu2015a} and sensitive place detection on the Life-logging dataset~\cite{Fan2016a}. Experimental results demonstrate that our framework performs significantly better than several existing privacy-preserving methods~\cite{Pathak2010a,Rajkumar2012a} in terms of classification accuracy. We also achieve comparable classification performance to some state-of-the-art visual recognition methods~\cite{He2015a,Liu2015a,Zeiler2014a}.

\subsection*{Related Work}
Privacy preservation has been actively studied in several areas including cryptography, statistics, machine learning, and data mining. One long-studied topic in cryptography is secure multiparty computation~\cite{Yao1986a}, which aims at computing some statistics (\eg, sum, maximum) securely over private data owned by multiple parties. More practical tasks include privacy-preserving data mining~\cite{Aggarwal2008a,Agrawal2000}, ubiquitous computing~\cite{Krumm2009a}, and social network analysis~\cite{Zhou2008a}.

One popular technique for preserving privacy is by perturbing outputs or intermediate results of algorithms based on the theory of differential privacy (DP)~\cite{Dwork2006a,Dwork2008a}. DP was classically introduced ``to release statistical information without compromising the privacy of the individual respondents~\cite{Dwork2008a},'' where the individual respondents are samples in a database and the statistical information is a certain statistic (e.g., average) of those samples. This can be done by adding properly-calibrated random noise to the statistical information such that one cannot distinguish the presence of an arbitrary single sample in the database. DP was then adapted and used also in the context of privacy-preserving machine learning, \eg, \cite{Chaudhuri2008a,Chaudhuri2011a,Chu2016a,Pathak2010a,Rajkumar2012a,Shokri2015a}. In the classification cases, a training dataset and classifier weights are referred to as a database and its statistical information, respectively. By adding properly-calibrated random noise to the classifier weights (a.k.a. output perturbation~\cite{Chaudhuri2011a}) or objective functions (objective perturbation), DP-based classification methods aim to prevent the classifier from leaking the presence of individual samples in the training dataset. While these methods are computationally efficient, the perturbed results are not the exact same as what could be originally learned from the given data. Moreover, the scale of noise typically increases exponentially to the level of privacy preservation (\eg, \cite{Dwork2006}). Perfect privacy preservation can never be achieved as long as one wishes to get some meaningful results from data. 

Another privacy-preserving technique is the use of homomorphic cryptosystems~\cite{OdedGoldreich2004}, which we will study in this work. The homomorphic cryptosystems also enable a variety of privacy-preserving machine learning and data mining~\cite{Bost2015a,Dowlin2016a,Graepel2012,Lindell2009a,Pathak2010a,Pinkas2002a,Takabi2016a,Vaidya2003a,Yuan2014a} because they can make some weights perfectly secret by encrypting them. Unlike the DP-based approaches, homomorphic encryption does not compromise the accuracy of algorithm outputs at the expense of its high computational cost. Our key contribution is to resolve this computational cost problem by introducing a new efficient encryption scheme.

Finally, studies on privacy preservation in computer vision are still limited. Recent work includes sensitive place detection~\cite{Fan2016a,Korayem2016a,Templeman2014a} and privacy-preserving activity recognition~\cite{Ryoo2017a}. These methods are designed to sanitize private information in a dataset and involve the potential trade-off between data security and utility. Another relevant topic is privacy-preserving face retrieval based on homomorphic encryption~\cite{Erkin2009a,Sadeghi2010a}. Because these methods encrypt all data samples, they can be applied only to thousands of data, while our approach can accept hundreds of thousands of data as input based on a distributed learning framework.

\section{Efficient Encryption Scheme for Privacy-Preserving Learning of Classifiers}
\label{sec:secsum}

The goal of this work is privacy-preserving learning of visual classifiers over distributed image data privately owned by people. As described earlier, our framework involves two types of parties: {\bf users} $\{U\ui{n}\mid n=1,\dots,N\}$ who individually own labeled image data and an {\bf aggregator} who exploits the data for learning classifiers while ensuring that no privacy information about the data is leaked during and after the learning procedure.

\subsection{Privacy-Preserving Learning Framework}
\label{subsec:framework}
We first describe how classifiers can be learned from distributed data. Let us denote classifier weights at step $t$ by $D$-dimensional vector $\bar{\vct{w}}_t\in\mathbb{R}^{D}$. We assume that the initial weight vector $\bar{\vct{w}}_0$ is provided by an aggregator, for example by using some public data already published on the web. As shown in Figure~\ref{fig:teaser}, (1) weight vector $\bar{\vct{w}}_t$ is first shared from an aggregator to users. (2) User $U\ui{n}$ updates $\bar{\vct{w}}_{t}$ locally with own labeled data and sends updated classifier $\vct{w}\ui{n}_t$ to the aggregator, and (3) the aggregator averages locally-updated classifiers for the initial weights of the next step: $\bar{\vct{w}}_{t+1} = \frac{1}{N}\sum_n \vct{w}\ui{n}_t$.

This framework, however, is vulnerable to potential privacy leakage at multiple places. An aggregator can have access to $\{\vct{w}\ui{n}_t\mid n=1,\dots, N\}$ each of which reflects user's private data. In addition, if a communication path between user and aggregator sides is not secure, $\vct{w}\ui{n}_t$ could also be intercepted by another user via man-in-the-middle attacks. If $\vct{w}\ui{n}_t$ is updated via stochastic gradient descent (SGD), the difference of two classifier weights $\vct{w}\ui{n}_t - \bar{\vct{w}}_{t}$ could be used to identify a part of the trained private data. Some concrete examples of how private data can be leaked are shown in our supplementary material.

Our new encryption scheme, the doubly-permuted homomorphic encryption (DPHE), is designed to prevent privacy leakage by homomorphically encrypting $\vct{w}\ui{n}_t$. By imposing sparse constraints on local classifier updates (\eg, \cite{Friedman2012,Tibshirani1994,Zou2005a}), the DPHE exploits the sparsity to encrypt high-dimensional classifiers efficiently. In Section~\ref{subsec:he}, we first explain briefly an algorithm to average local classifiers securely based on the Paillier cryptosystem~\cite{Paillier1999a}. Then Section~\ref{subsec:dperm} and Section~\ref{subsec:dperm2} present how DPHE `doubly-permutes' sparse data on a step-by-step basis. For simplicity, we will focus on one particular step by omitting time index $t$, that is, we consider the problem of computing $\bar{\vct{w}}=\frac{1}{N}\sum_n \vct{w}\ui{n}$ securely. 

\subsection{Secure-Sum with Paillier Encryption}
\label{subsec:he}

\begin{figure}[t]
\centering
\includegraphics[width=\linewidth]{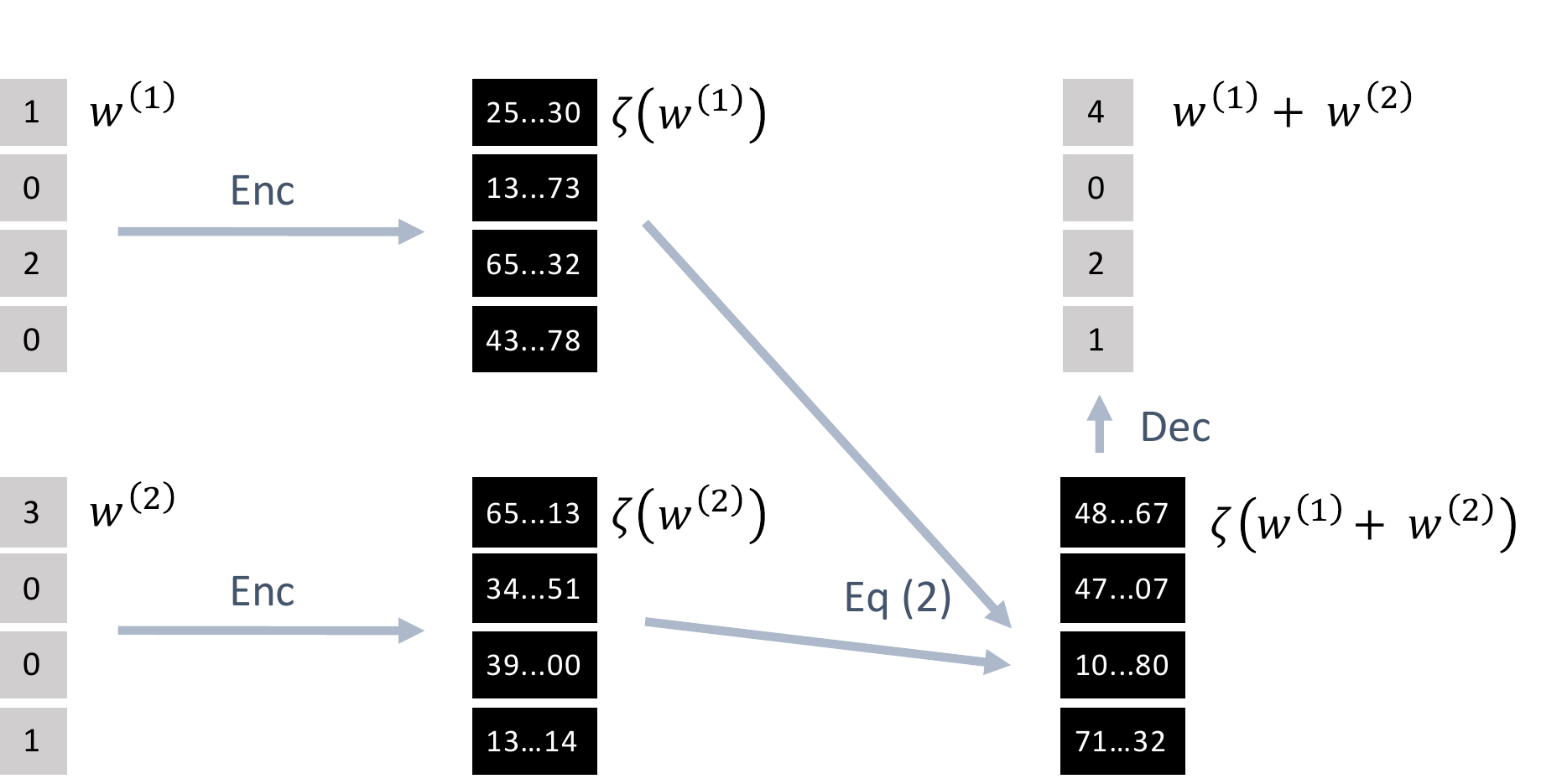}
\caption{{\bf Secure-Sum with Paillier Encryption}. Vectors $\vct{w}\ui{1}, \vct{w}\ui{2}$ are first encrypted with the Paillier cryptosystem. Ciphertexts $\hes{\vct{w}\ui{1}}, \hes{\vct{w}\ui{2}}$ are then used to compute encrypted sum $\hes{\vct{w}\ui{1} + \vct{w}\ui{2}}$. After the decryption we obtain $\vct{w}\ui{1} + \vct{w}\ui{2}$.}
\label{fig:enc1}
\end{figure}

Let us start from a simple secure-sum algorithm where each of $N$ users $\{U\ui{n}\mid n=1,\dots,N\}$ sends a single value $w\ui{n}\in\mathbb{R}$ to an aggregator, and the aggregator computes their sum $w\sub{sum}=\sum_n w\ui{n}$, provided that $w\ui{n}$ is never exposed to any other party but user $U\ui{n}$. To accomplish this, we utilize the Paillier cryptosystem~\cite{Paillier1999a}. Let us denote by $\he{w}$, the ciphertext of value $w\in\mathbb{R}$\footnote{Although the Paillier cryptosystem originally works only on a positive integer, we can deal with real values $w\in\mathbb{R}$ by scaling them to be a positive integer before encryption and rescaling them after decryption properly.}. As shown in \cite{Paillier1999a}, $\he{w}$ cannot be inverted to the original value $w$ without a decryption key. Also, because the encryption always involves generation of random numbers, multiple encryptions of the same values will result in different ciphertexts.

To enable this secure-sum algorithm, we introduce a {\bf key generator} who issues encryption and decryption keys and distributes the only encryption key to all of the parties. User $U\ui{n}$ then sends to the aggregator $\he{w\ui{n}}$ that is encrypted with the given encryption key. Using the Paillier encryption, the product of two ciphertexts results in a ciphertext of the two original plaintexts, \ie, $\he{w\ui{1}}\he{w\ui{2}}= \he{w\ui{1}+w\ui{2}}$\footnote{Specifically, Paillier encryption is defined as $\he{x}=g^xr^T {\rm mod}\; T^2$ where the pair of values $(g, T)$ is an encryption key and $r$ is a random number generated per encryption. Given $\he{x\ui{1}}=g^{x\ui{1}}(r\ui{1})^T {\rm mod}\; T^2$ and $\he{x\ui{2}}=g^{x\ui{2}}(r\ui{2})^T {\rm mod}\; T^2$, the encrypted sum is derived as follows: $\he{x\ui{1}}\he{x\ui{2}} = g^{x\ui{1}+x\ui{2}}(r\ui{1}r\ui{2})^T {\rm mod}\; T^2=\he{x\ui{1}+x\ui{2}}$. Further details can be found in \cite{Fontaine2007a, Paillier1999a}.}. Therefore, the aggregator can compute the encrypted sum of values as follows:
\begin{equation}
\he{w\sub{sum}} = \he{\sum_n w\ui{n}} = \prod_n\he{w\ui{n}}.
\label{eq:securesum}
\end{equation}
Finally, the aggregator asks the key generator for decrypting $\he{w\sub{sum}}$ and receives $w\sub{sum}$. As long as all the parties strictly follow this algorithm, privacy of $U\ui{n}$ can be preserved by restricting the aggregator's access to $w\ui{n}$.

The straightforward extension of this secure-sum algorithm for averaging locally-updated classifiers, $\vct{w}\ui{n}\in\mathbb{R}^D$, is to apply the Paillier encryption element-wise. In what follows, let $\vct{w}\sub{sum}=\bar{\vct{w}}N\in\mathbb{R}^D$ and let $\he{\vct{w}}$ be a $D$-dimensional vector which elements are individually Paillier encrypted. Similar to Equation~(\ref{eq:securesum}), $\vct{w}\sub{sum}$ can be computed securely as follows (see Figure~\ref{fig:enc1} for a two-vector case):
\begin{equation}
\he{\vct{w}\sub{sum}}=\he{\sum_n \vct{w}\ui{n}}=\odot_n\he{\vct{w}\ui{n}},
\label{eq:vecsum}
\end{equation}
where $\odot_n$ is the element-wise product of multiple vectors over index $n$. Then, the aggregator receives $\vct{w}\sub{sum}$ with help from the key generator and computes $\bar{\vct{w}} = \frac{1}{N}\vct{w}\sub{sum}$.

\subsection{Encryption with Single Permutation}
\label{subsec:dperm}
\begin{figure}[t]
\centering
\includegraphics[width=\linewidth]{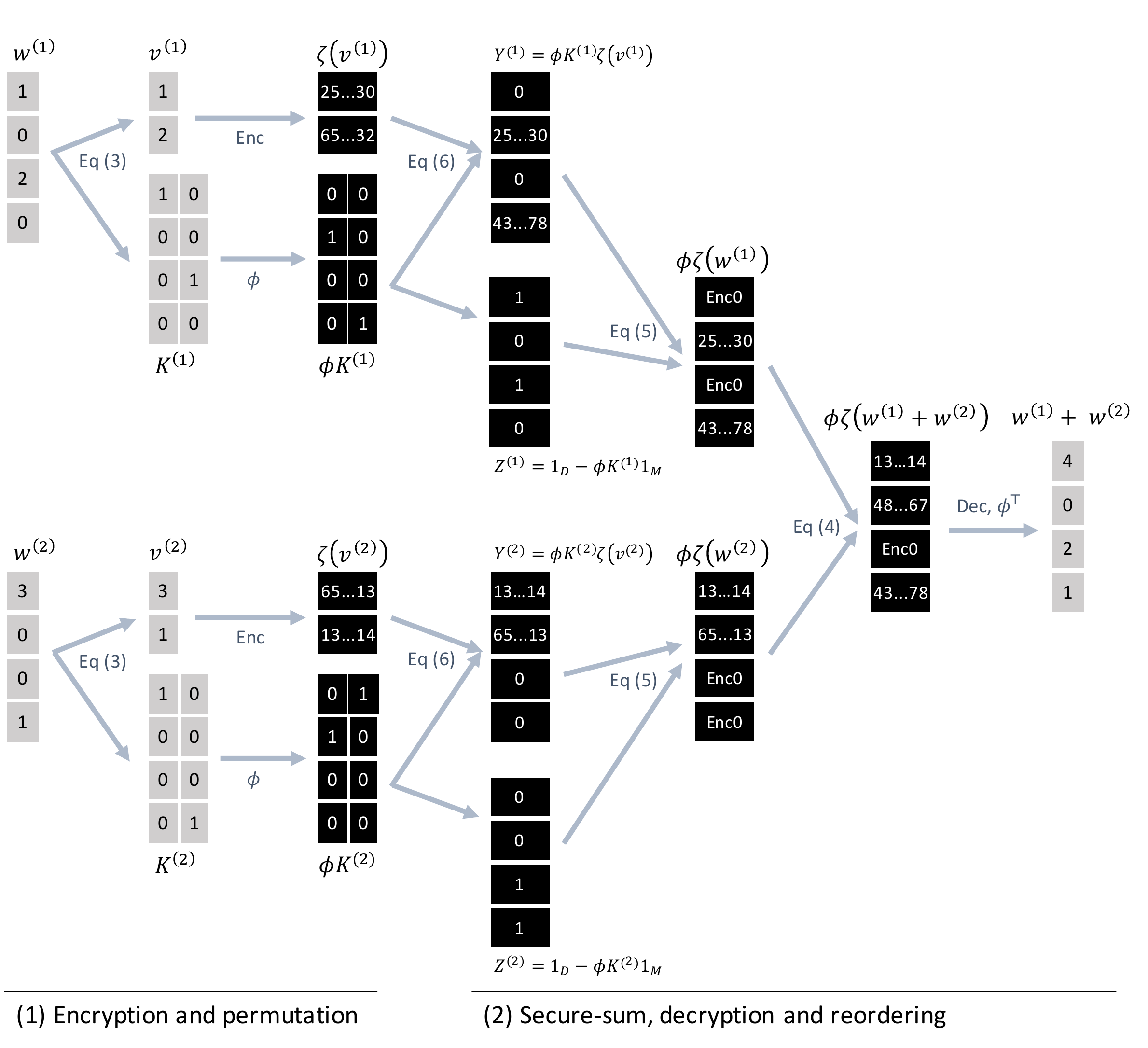}
\caption{{\bf Secure-sum with Efficient Encryption with Permutations}. (1) While applying the Paillier encryption only to $\vct{v}\ui{n}$, we shuffle index matrix $K\ui{n}$ with permutation matrix $\perm$ to keep it secure. (2) An aggregator can compute encrypted and shuffled sum $\phi\hes{\vct{w}\ui{1} + \vct{w}\ui{2}}$ and can ask a key generator for decryption and reordering to obtain $\vct{w}\ui{1} + \vct{w}\ui{2}$.}
\label{fig:sperm}
\end{figure}

The extension based on Equation~(\ref{eq:vecsum}) would, however, become infeasible for learning high-dimensional classifiers due to the high computational cost of homomorphic encryption. DPHE can overcome this problem by applying the homomorphic encryption only to a limited part of classifier weights that are updated with sparse constraints. 

Let us first introduce \emph{encryption capacity} $M$ that limits the number of values that are Paillier encrypted. We also denote by $\nz{\vct{w}\ui{n}}$, the actual number of non-zeros in $\vct{w}\ui{n}$. By choosing $M$ such that $\nz{\vct{w}\ui{n}}\leq M \leq D$, $\vct{w}\ui{n}$ can be decomposed into a pair of real-value vector $\vct{v}\ui{n}\in\mathbb{R}^{M}$ and binary matrix $K\ui{n}\in\{0, 1\}^{D\times {M}}$:
\begin{equation}
\vct{w}\ui{n} = K\ui{n}\vct{v}\ui{n},
\label{eq:decomposition}
\end{equation}
where $\vct{v}\ui{n}$ includes all the constituent non-zero values in $\vct{w}\ui{n}$. $K\ui{n}$ is an index matrix where the number of ones for each column is exact one and that for each row is at most one, and the $m$-th column of $K\ui{n}$ indicates the index of the $m$-th value of $\vct{v}\ui{n}$ in vector $\vct{w}\ui{n}$. Users encrypt the only $\vct{v}\ui{n}$ with the Paillier cryptosystem to obtain $\he{\vct{v}\ui{n}}$. If $\nz{\vct{w}\ui{n}}\ll D$, \ie, $\vct{w}\ui{n}$ is sparse, we can make this encryption more efficient by choosing smaller $M$. Note that there are multiple decompositions that satisfy Equation~(\ref{eq:decomposition}), and users can arbitrarily choose one of them. 

Now, the remaining problem is how to keep $K\ui{n}$ secret. This matrix indicates non-zero indices of $\vct{w}\ui{n}$ and could also be private information for users. For instance, when classifier weights take binary values (\eg, \cite{Cheng2014a,Courbariaux2015a,Hare2012a}), all non-zero values in $\vct{w}\ui{n}$ become one and the private information is found in $K\ui{n}$. To solve this problem, we propose shuffling $K\ui{n}$ with permutation matrix $\perm\in\{0, 1\}^{D\times D}$, where the number of ones for each row and for each column is exact one. This permutation matrix is 1) generated randomly by a key generator and 2) shared only among users, which we therefore refer to as \emph{user-shared (U-S) permutation matrix}. As illustrated in Figure~\ref{fig:sperm} (1), each user sends $\perm K\ui{n}$ to an aggregator. The aggregator cannot identify $K\ui{n}$ due to the absence of $\perm$ but is still able to compute the encrypted and shuffled sum of $\vct{w}\ui{n}$, namely $\perm\he{\vct{w}\sub{sum}}$, as follows (see also Figure~\ref{fig:sperm} (2)):
\begin{equation}
\perm\he{\vct{w}\sub{sum}}= \perm\he{\sum_n\vct{w}\ui{n}}=\odot_n \perm\he{\vct{w}\ui{n}}.
\label{eq:vecsum2}
\end{equation}
Here, $\perm\he{\vct{w}\ui{n}}$ is constructed as:
\begin{equation}
\perm\he{\vct{w}\ui{n}} = Y\ui{n} + {\rm Enc0}\cdot Z\ui{n},
\label{eq:hex}
\end{equation}
\begin{equation}
Y\ui{n}= \perm K\ui{n}\he{\vct{v}\ui{n}}, \;\; Z\ui{n}=\vct{1}_D - \perm K\ui{n}\vct{1}_M,
\label{eq:ViZi}
\end{equation}
where $\vct{1}_D, \vct{1}_M$ are all-one vectors of length $D$ and $M$. ${\rm Enc0}$ is a ciphertext of a zero computed \emph{only once} by the aggregator. Vector $Z\ui{n}$ adds ${\rm Enc0}$ to the locations of zeros in $Y\ui{n}$, which is necessary to compute $\perm\he{\vct{w}\sub{sum}}$ based on Equation~(\ref{eq:vecsum}).
Finally, the aggregator asks the key generator for reordering and decrypting $\perm\he{\vct{w}\sub{sum}}$ to obtain $\vct{w}\sub{sum}$.

\paragraph{Choosing encryption capacity $M$}
The most efficient approach to choose $M$ is by letting each user adjust $M$ independently such that $M\ui{n}=\nz{\vct{w}\ui{n}}$. However, this approach will reveal $\nz{\vct{w}\ui{n}}$ to an aggregator, which is less critical but still private information. Alternatively, one can keep $\nz{\vct{w}\ui{n}}$ secret as follows: First, the key generator determines $M$ such that $M \ll D$. If $\nz{\vct{w}\ui{n}}\leq M$ holds, a user encrypts $\vct{w}\ui{n}$ as presented above. Otherwise, the user arbitrarily splits $\vct{w}\ui{n}$ into several shards $\vct{w}\ui{n, 1},\dots,\vct{w}\ui{n, F}$ such that $\sum_f \vct{w}\ui{n, f} = \vct{w}\ui{n}$ and $\max_f\{\nz{\vct{w}\ui{n, f}}\mid f=1,\dots,F\}\leq M$. The user then encrypts and sends $\vct{w}\ui{n, f}$ one-by-one, and the aggregator receives them as if they are individual data. This modification does not affect output $\vct{w}\sub{sum}$ and is still efficient than naive encryption of all values as long as $\nz{\vct{w}\ui{n}} \ll D$.

\subsection{Encryption with Double Permutations}
\label{subsec:dperm2}

The U-S permutation matrix $\perm$ ensures that $K\ui{n}$ is secure against an aggregator. However, because $\perm$ is shared among all users, $K\ui{n}$ could be identified by a malicious user who intercepts $\perm K\ui{n}$ via man-in-the-middle attacks.

To address this issue, DPHE involves another set of permutation matrices to shuffle $K\ui{n}$: \emph{user-aggregator (U-A) permutation matrices} $\{\perm\ui{n}\in\{0, 1\}^{D\times D}\mid n= 1,\dots,N\}$ where $\perm\ui{n}$ is a permutation matrix defined in the same way as $\perm$ but \emph{is shared only between the user $U\ui{n}$ and the aggregator}. Both of the U-S and U-A permutation matrices are generated by the key generator and distributed properly. Then, each user \emph{doubly-permutes} $K\ui{n}$ as follows: 
\begin{equation}
\dperm{K\ui{n}} = \perm\ui{n}\perm K\ui{n}.
\label{eq:partial_reorder}
\end{equation}
Because the reordering of $\dperm{K\ui{n}}$ requires both $\perm$ and $\perm\ui{n}$, DPHE can now prevent an aggregator as well as any other user than $U\ui{n}$ from identifying $K\ui{n}$.

Importantly, because the aggregator knows all U-A permutation matrices $\{\perm\ui{n}\in\{0, 1\}^{D\times D}\mid n= 1,\dots,N\}$, $\dperm{K\ui{n}}$ can be reordered partially to obtain $\perm K\ui{n}$: 
\begin{equation}
\perm K\ui{n} = (\perm\ui{n})^\top \dperm{K\ui{n}}.
\end{equation}
Note that $(\perm\ui{n})^\top=(\perm\ui{n})^{-1}$. This allows the aggregator to compute $\perm\he{\vct{w}\sub{sum}}$ based on Equation~(\ref{eq:vecsum2}) by replacing $\perm K\ui{n}$ in $Y\ui{n}$ and $Z\ui{n}$ in Equation~(\ref{eq:ViZi}) with $(\perm\ui{n})^\top \dperm{K\ui{n}}$.
Algorithm~\ref{protocol1} summarizes how DPHE can average local classifiers securely and efficiently.

\begin{algorithm}[t]
\caption{Averaging Classifiers Securely with DPHE}
\label{protocol1}
\begin{algorithmic}[1]
\REQUIRE $N$ users $\{U\ui{n}\mid n=1,\dots,N\}$ who privately own $D$-dimensional classifier weights $\vct{w}\ui{n}$, aggregator $A$, key generator $G$, and encryption capacity $M$.
\ENSURE Averaged classifier $\bar{\vct{w}}=\frac{1}{N}\sum_n\vct{w}\ui{n}$
\STATE $G$ generates encryption key $\he{\cdot}$ and the corresponding decryption key, U-S permutation matrix $\perm$ and U-A permutation matrices $\{\perm\ui{n}\mid n=1,\dots,N\}$.
\STATE $G$ distributes $\he{\cdot},\perm,\perm\ui{n}, M$ to $U\ui{n}$ and $\he{\cdot},\{\perm\ui{n}\mid n=1,\dots,N\}$ to $C$.
\STATE $U\ui{n}$ decomposes $\vct{w}\ui{n}$ into $\vct{v}\ui{n}$ and $K\ui{n}$. 
\STATE $U\ui{n}$ sends $\he{\vct{v}\ui{n}}, \dperm{K\ui{n}}$ to $A$.
\STATE $A$ computes $\perm\he{\vct{w}\ui{n}}$ for each $n$ to get $\perm\he{\vct{w}\sub{sum}}$. 
\STATE $A$ asks $G$ for decrypting and reordering $\perm\he{\vct{w}\sub{sum}}$ to receive $\vct{w}\sub{sum}$, and computes $\bar{\vct{w}}=\frac{1}{N}\vct{w}\sub{sum}$.
\end{algorithmic}
\end{algorithm}

\subsection{Security Evaluation}
\label{sec:proof}
This section briefly describes that DPHE is guaranteed to be secure under certain assumptions. A more formal evaluation is present in our supplementary material.

\paragraph{Assumptions}
Recall that our framework involves users, an aggregator, and a key generator. Our security evaluation is built upon one of the classical assumptions in cryptography that they are all \emph{semi-honest} --- each party ``follows the protocol properly with the exception that it keeps a record of all its intermediate computations''~\cite{OdedGoldreich2004}. For instance, the aggregator is not allowed to ask the key generator to decrypt individual classifier $\perm\he{\vct{w}\ui{n}}$ while he may use $\bar{\vct{w}}$ to identify $\vct{w}\ui{n}$. We also assume that there is no collusion among the parties. For example, we will not consider cases where the aggregator and the key generator collude to share a Paillier decryption key and where the aggregator and all but one users collude to collect $\{\vct{w}\ui{j}\mid j\neq n\}$ to reveal $\vct{w}\ui{n}$. These assumptions are justified in some practical applications such as crowdsourcing~\cite{Kajino2014a} and multiparty machine learning~\cite{Pathak2010a}. Finally, as described in Section~\ref{subsec:framework}, we consider a case where malicious users may intercept data sent from other users to the aggregator by slightly abusing the semi-honesty assumption.

\paragraph{Security on Algorithm~\ref{protocol1}}
With the assumptions above, we can guarantee that \emph{no one but user $U\ui{n}$ can identify $\vct{w}\ui{n}$ and its non-zero indices $K\ui{n}$ during and after running Algorithm~\ref{protocol1} if the number of users satisfies $N\geq 3$}. Specifically, non-zero weights of $\vct{w}\ui{n}$, $\vct{v}\ui{n}$, are encrypted with the Pailler cryptosystem, which prevents an aggregator and all users from identifying $\vct{v}\ui{n}$ from its ciphertext $\he{\vct{v}\ui{n}}$ as they do not have a decryption key. Also, they cannot identify $K\ui{n}$ from its doubly-permuted form $\dperm{K\ui{n}}$ due to the lack of either U-S permutation matrix $\perm$ or U-A matrix $\perm\ui{n}$. Although the key generator owns the decryption key and all of the permutation matrices, he does not access $\vct{v}\ui{n}$ and $K\ui{n}$ in the algorithm. Finally, $\bar{\vct{w}}$ cannot be used to identify $\vct{w}\ui{n}$ when $N\geq 3$. For example, when $N=3$, user $U\ui{1}$ gets $\bar{\vct{w}}=\frac{1}{3}(\vct{w}\ui{1} + \vct{w}\ui{2} + \vct{w}\ui{3})$. However, $U\ui{1}$ can only compute $3\bar{\vct{w}} - \vct{w}\ui{1} = \vct{w}\ui{2} + \vct{w}\ui{3}$, which still remains ambiguous to identify $\vct{w}\ui{2}$ and $\vct{w}\ui{3}$. Similarly, non-zero indices of $\bar{\vct{w}}$ is just the union of those of $\vct{w}\ui{n}$ and not useful to reveal $K\ui{n}$. This security on $\bar{\vct{w}}$ holds also for the aggregator and the key generator.



\paragraph{Limitations} The requirement of $N\geq 3$ implies that, if only one or two users are available at one time, the aggregator will never publish classifiers without revealing each user's updates. Moreover, we cannot prevent certain attacks using published classifiers to infer potentially-private data, \eg, using a face recognition model and its output to reconstruct face images specific to the output~\cite{Fredrikson2015}, although such attacks are not currently able to identify which users privately owned the reconstructed data in a distributed setting.

\section{Experiments}
\label{sec:exp}
In this section, we address several visual recognition tasks with our privacy-preserving learning framework based on DPHE. Specifically, we first evaluate DPHE empirically under various conditions systematically with object classification tasks on Caltech101~\cite{Fei-Fei2007a} and Caltech256~\cite{Griffin2007a} datasets. Then we tackle more practical and sensitive tasks: face attribute recognition on the CelebA dataset~\cite{Liu2015a} and sensitive place detection on the Life-logging dataset~\cite{Fan2016a}.

\subsection{Settings of Experiments}
\label{subsec:exp_setting}
Throughout our experiments, we learned a linear SVM via SGD. We employed the elastic net~\cite{Zou2005a}, \ie, the combination of L1 and L2 regularizations, to enforce sparsity on locally-updated classifiers. For a simulation purpose, multiple users, an aggregator, and a key generator were all implemented in a single computer and thus no transmissions among them were considered in the experiments.

As a baseline, we adopted the following two off-the-shelf privacy-preserving machine learning methods. Similar to our approach, these methods were designed to learn classifiers over distributed private data while ensuring that no private information about distributed data was leaked during and after the learning procedure.
\begin{description}
\item[PRR10~\cite{Pathak2010a}] Users first train $D$-dimensional linear classifiers locally using their own data. They then average the locally-trained classifiers followed by adding a $D$-dimensional Laplace noise vector to the output classifier based on differential privacy. Because the scale of noise depends on the size of local user data and thus involves private information, the noise vector is encrypted with the Paillier cryptosystem (\ie, this method needs $D$-times encryptions). An L2-regularized logistic regression classifier was trained by each user via SGD.
\item[RA12~\cite{Rajkumar2012a}] Similar to our approach, RA12 iteratively averages locally-updated classifiers to get a better classifier. Unlike our method and PRR10, each local classifier is updated in a gradient descent fashion while adding the combination of Laplace and Gaussian noises to its objective function to prevent learned classifiers from leaking private information based on differential privacy. We employed an L2-regularized linear SVM for each of local classifiers.
\end{description}
These methods involve hyper-parameters $\epsilon\in\mathbb{R}_+, \delta\in\mathbb{R}_+$ to control the strength of privacy preservation. We followed the original papers~\cite{Pathak2010a,Rajkumar2012a} and set to $\epsilon=0.2, \delta=0.05$, which were adjusted and justified in the papers to preserve privacy while keeping high recognition performances. 

We set $N=5$ for these two baselines and our method, that is, five users were assumed to participate in a task. Encryption capacity $M$ was set to $M=\lceil 0.1D\cdot{\it NC}\rceil$ , where $D$ was the dimension of learned classifiers and ${\it NC}$ was the number of classes. If each of locally-updated classifiers has more than 90\% sparsity, its weights can be encrypted and transmitted at one time (see Section~\ref{subsec:dperm}).

\subsection{Performance Analysis on Caltech101/256}
\label{subsec:exp_caltech}
We compared DPHE against the privacy-preserving baselines~\cite{Pathak2010a,Rajkumar2012a} as well as some of the state-of-the-art object classification methods~\cite{He2015a,Zeiler2014a} on Caltech101/256 object classification datasets. Our evaluation was based on the protocol shown in \cite{He2015a, Zeiler2014a} as follows. For Caltech101 dataset~\cite{Fei-Fei2007a}, we generated training and testing data by randomly picking 30 and no more than 50 images respectively for each category. On the other hand, for the Caltech256 dataset~\cite{Griffin2007a}, we chose random 60 images per category for training data and the rest of images for testing data. 

Recall that our learning framework requires some initialization data for obtaining the initial weights $\bar{\vct{w}}_0$. Therefore, we first 1) left 10\% of the training data for the initialization data and 2) split 90\% of them into five subsets of the same size, which we regarded as private image data of five users. For the two baselines~\cite{Pathak2010a,Rajkumar2012a}, we added 1) to each of 2) to serve as private data of each user. This ensured that in both of our method and the baselines each user had access to the same image set. We ran the evaluation ten times to get a mean and a standard deviation of classification accuracies averaged over all the categories.

For our method and the baselines, we extracted 2048-dimensional outputs of the global-average pooling layer of the deep residual network~\cite{He2016a} with 152 layers trained on ImageNet~\cite{Russakovsky2015a} as a feature, which was then normalized to have zero-mean and unit-variance for each dimension. In our method, the relative strength of L1 regularization in the elastic net was fixed to 0.5, and the overall strength of elastic-net regularization was adjusted adaptively so that local classifiers had more than 90\% sparsity on average. 

\paragraph{Comparisons to other methods}
\begin{table}[t]
\caption{{\bf Object Recognition on Caltech101/256}: mean and standard deviation of classification accuracies (\%) averaged over object categories. Note that the scores of HZRS14 and ZF13 were cited from their original papers~\cite{He2015a,Zeiler2014a}.}
\centering
\scalebox{.9}{
\begin{tabular}{cccc}
\toprule
     Methods &  Caltech101& Caltech256 & Privacy\\ 
     \midrule
     HZRS14~\cite{He2015a}& $93.4\pm 0.5$& N/A &\xmark\\
     ZF13~\cite{Zeiler2014a}& $85.4\pm 0.4$& 72.6$\pm 0.1$&\xmark\\
     PRR10~\cite{Pathak2010a}&$41.6\pm 1.2 $&$55.9\pm 0.5$&\cmark \\
     RA12~\cite{Rajkumar2012a}&$83.8\pm 1.1 $&$68.0\pm 0.3$&\cmark \\
     \midrule
     {\bf DPHE}& $89.3\pm 0.8$&$74.7\pm 0.4$ &\cmark \\
\bottomrule
\end{tabular}
}
\label{tab:caltech1}
\end{table}
Table~\ref{tab:caltech1} presents classification accuracies of all the methods. Our method performed almost comparably well to state-of-the-art methods~\cite{He2015a,Zeiler2014a} while preserving privacy. We also found that DPHE outperformed the privacy-preserving baselines~\cite{Pathak2010a,Rajkumar2012a}. We here emphasize that DPHE was designed to preserve privacy without adding any noise on locally-updated or averaged classifiers, which would be one main reason of our performance improvements over the two baselines. RA12~\cite{Rajkumar2012a} inevitably needs to add noise on local classifiers, and still does not preserve privacy perfectly due to the trade-off between data utility and security. On the other hand, PRR10~\cite{Pathak2010a} achieves better privacy preservation thanks to the combination of homomorphic encryption and perturbed classifiers. However, this baseline did not employ iterative updates of classifiers unlike other methods and thus could not effectively leverage distributed training data.

\paragraph{Encryption time}
\begin{table}[t]
\caption{{\bf Sparsity (\%), Classification Accuracy (\%), and Encryption Time (sec) of DPHE on Caltech101.}}
\centering
\scalebox{.9}{
\begin{tabular}{c|cc}
\toprule
     Sparsity  & Accuracy & Time \\
     \hline
     0.01 & 89.7 & 620 \\
     79.1 & 89.6 & 186 \\
     95.6 & 88.2 & 62 \\
\bottomrule
\end{tabular}
}
\label{tab:caltech2}
\end{table}
Table~\ref{tab:caltech2} depicts the relationship between the sparsity of locally-updated classifiers (averaged over all users and all iterations), classification accuracies, and encryption times of DPHE on Caltech101. We adjusted the sparsity by changing the relative strength of L1 regularization in the elastic net. Overall, the increase of sparsity little affected the classification performance. The comparison between the top and the bottom rows of the table indicates that we achieved ten times faster encryption at the cost of the only 1.5 percent points decrease of classification accuracies. Note that PRR10~\cite{Pathak2010a} requires to encrypt a 2048-dimensional dense noise vector for privacy preservation, which took about 620 sec in this condition.

\paragraph{Number of users}
\begin{figure}[t]
\centering
\includegraphics[width=\linewidth]{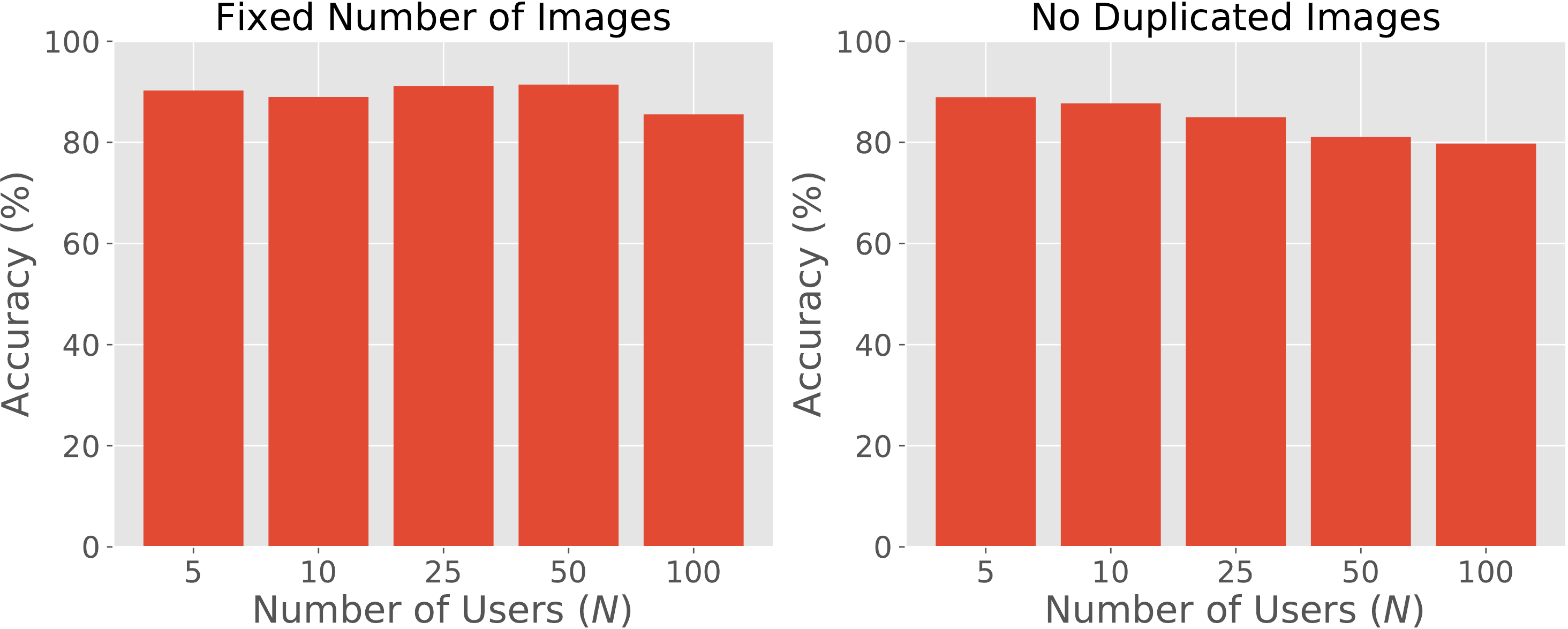}
\caption{{\bf Classification Accuracy (\%) with Different Numbers of Users on Caltech101.} Left: users own the fixed number of partially-duplicated images. Right: users own the decreasing number (inversely proportional to $N$) of non-duplicated images.}
\label{fig:N}
\end{figure}
We also investigate how the number of users $N$ affected classification performances on Caltech101 under two different conditions. In the left of Figure~\ref{fig:N}, we repeated multiple (1, 2, 5, 10, and 20) times the random split of training data into five subsets. This yielded $N=5, 10, 25, 50, 100$ subsets each of which has the fixed number of partially duplicated images. We found that the increase of user sizes hardly affected classification performances but slightly made it worse when $N=100$. On the right of the figure, we split the training data into $N=5, 10, 25, 50, 100$ subsets. While this split ensured that there were no duplicated images across multiple subsets, the number of images in each subset decreased as $N$ became larger. The result with this condition implies that DPHE requires each user to own a sufficient number of images to make each local update stable.

\subsection{Practical Examples}
\label{subsec:exp_practical}
As proof-of-concept applications, this section addresses two practical visual learning tasks where training images could be highly private and need privacy preservation: face attribute recognition and sensitive place detection. Note that the following experiments are practical also in terms of data size: the two datasets used in the experiments contained more than one hundred thousand images in total.

\paragraph{Face attribute recognition}
We first adopted a face attribute recognition task on the CelebFaces Attributes (CelebA) Dataset~\cite{Liu2015a} comprising 202,599 face images with 40 attribute labels (\eg, big lips, oval faces, young). As a feature, we used 512-dimensional outputs of the FC5 layer of the face recognition network~\cite{Wen2016b} trained on the CASIA WebFace~\cite{Yi2014}. All face images were cropped and aligned using provided facial landmarks. We used validation data (19,868 images) for initialization data and split training data (162,770 images) into five subsets to run our method with $N=5$. Similar to the previous experiments, we added the initialization data to each of the five subsets to serve as one of the five private data for privacy-preserving baselines~\cite{Pathak2010a,Rajkumar2012a}. The regularization strength of our method was adjusted so that the sparsity of local classifiers was about 75\% on average, which required about 14 seconds for encrypting each classifier. Table~\ref{tab:celeba} reports a recognition accuracy averaged over all of the 40 face attributes. DPHE worked comparably well to the state-of-the-art method~\cite{Liu2015a} and outperformed privacy-preserving baselines~\cite{Pathak2010a,Rajkumar2012a}.

\begin{table}[t]
\caption{{\bf Face Attribute Recognition on CelebA Dataset} : mean accuracies (\%) over 40 face attribute labels. The score of LLWT15 was cited from their original paper~\cite{Liu2015a}.}
\centering
\scalebox{.9}{
\begin{tabular}{ccc}
\toprule
    Methods & Accuracy & Privacy \\
     \midrule
     LLWT15~\cite{Liu2015a}& 87 & \xmark\\
     PRR10~\cite{Pathak2010a}& 78 & \cmark\\
     RA12~\cite{Rajkumar2012a}& 64 & \cmark\\
     \midrule
     {\bf DPHE} & 84 & \cmark\\
\bottomrule
\end{tabular}
}
\label{tab:celeba}
\end{table}

\paragraph{Sensitive place detection}
Another interesting application where privacy preservation plays an important role is sensitive place detection presented in~\cite{Fan2016a,Korayem2016a,Templeman2014a}. This task aims at detecting images of a place where sensitive information could be involved, such as those of bathrooms and laptop screens. We used the Life-logging dataset~\cite{Fan2016a} that contained 131,287 training images selected from the Microsoft COCO~\cite{Lin2014b} and Flickr8k~\cite{Hodosh2013} datasets, and 7,210 testing images recorded by a wearable life-logging camera. Following~\cite{Fan2016a}, images with specific annotations: `toilet,' `bathroom,' `locker,' `lavatory,' `washroom,' `computer,' `laptop,' `iphone,' `smartphone,' and `screen' were regarded as sensitive ones. We used a 4096-dimensional deep feature provided in the work~\cite{Fan2016a}. 132 images containing both sensitive and non-sensitive samples were chosen randomly for initialization data and the rest was split into five subsets to serve as private data. The regularization strength was chosen such that sparsity of local updates was more than 95\% on average, which required only 1.2 seconds for encryption. Table~\ref{tab:deepdiary} presents average precision scores for detecting sensitive places with our method and the two privacy-preserving baselines. Note that the method in the original paper~\cite{Fan2016a} could not be compared directly due to the absence of its average precision score and different experimental conditions (\ie, smaller numbers of training and testing images were used). We confirmed that DPHE performed fairly well compared to the two baselines.

\begin{table}[t]
\caption{{\bf Sensitive Place Detection on Life-logging Dataset}: average precision for detecting images of sensitive places.}
\centering
\scalebox{.9}{
\begin{tabular}{cccc}
\toprule
    Methods & PRR10~\cite{Pathak2010a} & RA12~\cite{Rajkumar2012a} & {\bf DPHE} \\
     \midrule
     Average Precision & 0.546 & 0.704 & 0.729 \\
\bottomrule
\end{tabular}
}
\label{tab:deepdiary}
\end{table}

\section{Conclusion}
We developed a privacy-preserving framework with a new encryption scheme DPHE for learning visual classifiers securely over distributed private data. Our experiments show that the proposed framework outperformed other privacy-preserving baselines \cite{Pathak2010a,Rajkumar2012a} in terms of accuracy and worked comparably well to several state-of-the-art visual recognition methods.

Although we focused exclusively on the learning of linear classifiers, DPHE can encrypt any type of sparse high-dimensional data efficiently and thus could be used for other tasks. One interesting task is distributed unsupervised or semi-supervised learning which will not require users to annotate all private data. Another promising direction for the future work is learning much higher-dimensional models like sparse convolutional neural networks~\cite{BaoyuanLiu2015,Sun2016a}. Our privacy-preserving framework will make it easy to provide diverse data sources for learning such complex models.

\appendix

\section{Supplementary Material}
\subsection{Some Statistics on Public/Private Images}
We are interested in leveraging `private' images, which are not shared publicly but just saved on a personal storage privately, for visual learning. In this section, we would like to provide some statistics that motivate our work.

Based on the recent report from Kleiner Perkins Caufield \& Byers~\cite{Meeker2016}, the number of photos shared publicly on several social networking services (Snapchat, Instagram, WhatsApp, Facebook Messenger, and Facebook) per day has reached almost 3.5 billion in 2015. It also shows that the number of smartphone users in the world was about 2.5 billion in the same year. From these statistics, if everyone takes three photos per day on average, about four billion photos in total would be stored privately everyday. Some prior work~\cite{Hoyle2015,Hoyle2014} has shown that such private photos still contained meaningful information including people, faces, and written texts, as well as some sensitive information like a computer screen and a bedroom. Our privacy-preserving framework is designed to learn visual classifiers by leveraging this vast amount of private images while preserving the privacy of the owners.

\subsection{Examples of Privacy Leakage from Locally-Updated Classifiers}

In our framework, users update classifier weights $\barw_t\in\mathbb{R}^D$ locally using their own private data and send the updated ones $\wui_t\in\mathbb{R}^D$ to the aggregator. Here we discuss how the combination of $\barw_t$ and $\wui_t$ can be used to reveal a part of the trained data.

Let us denote a labeled sample by $z_i=(\vct{x}_i, y_i)$ where $\vct{x}_i\in\mathbb{R}^D$ is a feature vector and $y_i\in\{-1, 1\}$ is a binary label. The whole data privately owned by a single user is then described by $\mathcal{Z}=\{z_i\mid i=1,\dots,K\}$. In order to learn a classifier, we minimize a regularized loss function which is defined with weights $\vct{w}$ and data $\mathcal{Z}$ as follows:
\begin{equation}
Q(\mathcal{Z}, \vct{w})=\ell (\mathcal{Z},\vct{w}) + \lambda R(\vct{w}),
\end{equation}
where $\ell (\mathcal{Z}, \vct{w})$ is a loss function, $R(\vct{w})$ is a certain regularizer, and $\lambda$ is a regularization strength. 

\subsubsection{Stochastic Gradient Descent}
As described in Section 2.1 of the original paper, a part of trained private data could be leaked when users update a classifier via stochastic gradient descent (SGD). With SGD, users obtain weights $\wui_t$ by updating $\barw_t$ based on the gradient of regularized loss with respect to single sample $z_t=(\vct{x}_t, y_t)\in \mathcal{Z}$ picked randomly from $\mathcal{Z}$~\cite{Bottou2012}:
\begin{equation}
    \wui_t = \barw_t -\gamma_t \nabla_{\barw_t}Q(z_t, \barw_t),
    \label{eq:sgd}
\end{equation}
where $\gamma_t$ is a learning rate at time step $t$ and controls how much one can learn from the sample $z_t$. Loss gradient $\nabla_{\barw_t}Q(z_t, \barw_t)$ is described as follows: 
\begin{equation}
    \nabla_{\barw_t}Q(z_t, \barw_t) = \nabla_{\barw_t}\ell (z_t, \barw_t) + \lambda\nabla_{\barw_t} R(\barw_t).
    \label{eq:gradient}
\end{equation}
By plugging Equation~(\ref{eq:gradient}) into Equation~(\ref{eq:sgd}), we obtain:
\begin{equation}
\nabla_{\barw_t}\ell(z_t, \barw_t) = \frac{\barw_t - \wui_{t+1}}{\gamma_t} - \lambda \nabla_{\barw_t}R(\barw_t).
\label{eq:get_grad}
\end{equation}

Now we are interested in what one can know about $z_t = (\barw_t, y_t)$ from Equation~(\ref{eq:get_grad}) \emph{where both $\barw_t$ and $\wui_{t+1}$ are given}. If the type of regularizer $R(\cdot)$ can be identified (\eg, L2 regularization), $\nabla_{\barw_t}R(\barw_t)$ can be computed exactly. In addition, if concrete parameters for $\gamma_t$ and $\lambda$ can be guessed (\eg, when using a default parameter of open source libraries) and if a specific loss function is used for $\ell(z_t, \barw_t)$, one can narrow down the private sample $z_t$ to several candidates.

Specifically, let $\Theta$ be the RHS of Equation~(\ref{eq:get_grad}), which is given when $R(\cdot)$ is identified and $\gamma_t, \lambda$ is estimated. Then, when using some specific loss functions, we can solve $\nabla_{\barw_t}\ell(z_t, \barw_t)=\Theta$ for $z_t$ as follows: 
\begin{description}
\item[Hinge loss] $\nabla_{\barw_t}\ell(z_t, \barw_t) = \vct{0}_D$ (an all-zero vector of size $D$) if $1-y_t\barw_t^\top\vct{x}_t < 0$ or $-y_t\vct{x}_t$ otherwise. If $\Theta$ is a non-zero vector, then $z_t=(\Theta, -1)$ or $(-\Theta, 1)$.
\item[Logistic loss] $\nabla_{\barw_t}\ell(z_t, \barw_t) = \frac{-y_t\vct{x}_t}{1 + \exp(y_t\barw_t^\top\vct{x}_t)}=\frac{-X}{1+\exp(\bar{w}_t^\top X)}$, where $\barw_t$ is known and $X=y_t\vct{x}_t$. $X$ can be obtained numerically (\eg, via the Newton's method), and $z_t=(X, 1)$ or $(-X, -1)$.
\end{description}

Note that this problem of sample leakage can happen also when users have just a single image in their storage. 

\subsubsection{Gradient Descent}
When users have more than one image, it might be natural to use a gradient descent (GD) technique instead of SGD. Namely, we evaluate the loss gradient averaged over the whole data $\nabla_{\barw_t}\ell(\mathcal{Z}, \barw_t)=\frac{1}{K}\sum_i\nabla_{\barw_t}\ell(z_i, \barw_t)$ instead of that of a single sample. Although this averaging can prevent one from identifying individual sample $z_i$, the equation $\nabla_{\barw_t}\ell(\mathcal{Z}, \barw_t)=\Theta$ still reveals some statistics of private data $\mathcal{Z}$. Specifically, if the class balance of data is extremely biased, one can guess an average of samples that were not classified correctly. One typical case of class unbalance arises when learning a detector of abnormal events. This will regard most of training samples as negative ones.

Let us consider an extreme case where all of the samples owned by a single user belong to the negative class, \ie, $y_i=-1\;\forall z_i=(\vct{x}_i, y_i) \in \mathcal{Z}$. If we use the hinge loss, a set of samples that were not classified perfectly is described by $\bar{\mathcal{Z}}=\{z_i=(\vct{x}_i,y_i)\mid y_i\barw_t^\top\vct{x}_i<1\}\subseteq \mathcal{Z}$. Then, the averaged loss gradient is transformed as follows:
\begin{equation}
\nabla_{\barw_t}\ell(\mathcal{Z}, \barw_t)= \frac{1}{K}\sum_{z_i\in\bar{\mathcal{Z}}} -y_i\vct{x}_i=\frac{1}{K}\sum_{z_i\in\bar{\mathcal{Z}}}\vct{x}_i=\Theta.
\end{equation}
Namely, $\Theta$ is proportional to the average of samples that were not classified correctly.

For the logistic loss, when $y_i=-1$, the loss gradient with respect to a single sample becomes $\nabla_{\barw_t}\ell(z_i, \barw_t) = \frac{\vct{x}_i}{1 + \exp(-\barw_t^\top\vct{x}_i)}=P(y_i=1\mid \vct{x}_i)\vct{x}_i$, where $P(y_i=1\mid\vct{x}_i)$ is close to $0$ when $z_i$ is classified correctly as negative, and increases up to $1$ when classified incorrectly as positive. Then, the averaged loss gradient becomes:
\begin{equation}
\nabla_{\barw_t}\ell(\mathcal{Z}, \barw_t)=\frac{1}{K}\sum_{i}P(y_i=1\mid\vct{x}_i)\vct{x}_i=\Theta.
\end{equation}
If all the samples are classified confidently, \ie, $|\barw_t^\top\vct{x}_i| \gg 0 \; \forall z_i\in\mathcal{Z}$, $\Theta$ is again proportional to the average of samples not classified correctly.

\subsubsection{Reconstructing Images from Features}
The previous sections demonstrated that classifiers updated locally via SGD/GD could expose a part of trained feature vectors. We argue that users will further suffer from a higher privacy risk when the features could be inverted to original images~(\eg, \cite{Dosovitskiy2016a,Mahendran2015,Vondrick2013a}).  Figure~\ref{fig:recon} showed examples on image reconstruction from features using \cite{Dosovitskiy2016a} on some images from Caltech101~\cite{Fei-Fei2007a}. We extracted outputs of the fc6 layer of the Caffe reference network used in the open source library\footnote{\url{http://lmb.informatik.uni-freiburg.de/resources/software.php}}. Although reconstructed images do not currently describe the fine-details of original images (\eg, contents displayed on the laptop), they could still capture the whole picture indicating what were recorded or where they were recorded.

\begin{figure}[t]
\centering
\includegraphics[width=\linewidth]{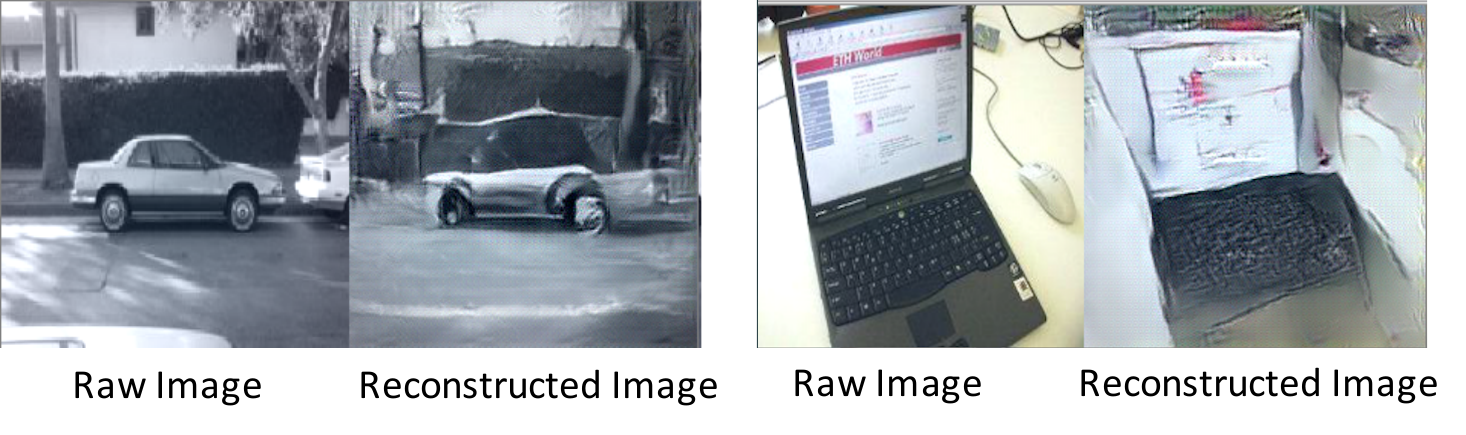}
\caption{{\bf Image Reconstruction from Features} with~\cite{Dosovitskiy2016a}}
\label{fig:recon}
\end{figure}

\subsection{Additional Experimental Results}
In the original paper, we evaluated our approach on a variety of tasks not only object classification on the classic Caltech Datasets~\cite{Fei-Fei2007a,Griffin2007a} but also face attribute recognition and sensitive place detection on a large-scale dataset~\cite{Fan2016a,Liu2015a}. This section introduces some additional experimental results using different tasks or datasets.

\paragraph{Video Attribute Recognition on YouTube8M Subset}
We evaluated the proposed method ({\bf DPHE}) as well as the two privacy-preserving baselines ({\bf PPR10}~\cite{Pathak2010a}, {\bf RA12}~\cite{Rajkumar2012a}) on a video attribute recognition task using a part of YouTube8M dataset~\cite{Sami2016}. Specifically, we picked 227,476 videos from the training set and 79,398 videos from the validation set. Similar to the data preparation in Section 3.2 of the original paper, 50,000 videos of our training set were left for the initialization data and the rest was split into five to serve as private data with $N=5$. Although over 4,000 attributes like `Games,' `Vehicle,' and `Pina Records.' were originally annotated to each video, our evaluation used the top 100 frequent attributes that were annotated to more than 1,000 videos in our training set. For each video, we extracted outputs of the global average pooling layer of the deep residual network~\cite{He2016a} trained on ImageNet~\cite{Russakovsky2015a} every 30 frame (about once in a second) and average them to get a single 2048-dimensional feature vector. Table~\ref{tab:yt8m} describes the mean average precision (mAP) for the top 10, 50, and 100 frequently-annotated attributes. We confirmed that DPHE outperformed the two privacy-preserving baselines. The sparsity of locally-updated classifiers was 65\% on average, which resulted in about 3.5 minutes for the encryption.  In order to see the original performance obtained by using residual network features, we introduced another baseline ({\bf No-PP} in the table) that learned an L2-regularized linear SVM on the whole training data via SGD. The results demonstrated that the performance with DPHE was almost comparable to that with No-PP.

\begin{table}[t]
\caption{{\bf Video Attribute Recognition Results on the Part of YouTube8M Dataset}: mean average precision (mAP) for the top 10, 50, and 100 frequently-annotated attributes.}
\centering
\scalebox{.9}{
\begin{tabular}{ccccc}
\toprule
    Methods & mAP ($\sim$10) & ($\sim$50) & ($\sim$100) & Privacy \\
     \midrule
     PRR10~\cite{Pathak2010a}& 0.62 & 0.45 & 0.38 & \cmark\\
     RA12~\cite{Rajkumar2012a}& 0.60 & 0.45 & 0.37 & \cmark\\
     No-PP & 0.70 & 0.55 & 0.48 & \xmark\\
     \midrule
    {\bf DPHE}& 0.70 & 0.53 & 0.47 & \cmark\\
\bottomrule
\end{tabular}
}
\label{tab:yt8m}
\end{table}

\paragraph{Clustering on Caltech101/256}
\label{subsec:exp_cluster}
Unlike the other privacy-preserving baselines~\cite{Pathak2010a,Rajkumar2012a}, DPHE can also be applied to an unsupervised clustering task based on mini-batch k-means~\cite{Sculley2010a}. Instead of learning a classifier with the initialization data, an aggregator first runs k-means++~\cite{Arthur2007a} to distribute cluster centroids to users. Users then update the centroids locally with sparse constraints. We used the L1-regularized stochastic k-means~\cite{Jumutc2015a} to obtain the sparse centroids. Regularization strength ($\eta$ in \cite{Jumutc2015a}) was chosen adaptively so that the sparsity of cluster centroids was more than 90\% on average. Table~\ref{tab:clustering} shows an adjusted mutual information score of the clustering task on Caltech101~\cite{Fei-Fei2007a} and Caltech256~\cite{Griffin2007a}. As a baseline method, we chose a standard mini-batch k-means (S10~\cite{Sculley2010a}). For all of the methods, the number of image categories was given as the number of clusters, \ie, we assumed that the correct cluster number was known. We found that DPHE achieved a comparable performance to the baseline method.

\begin{table}[t]
\caption{{\bf Clustering Results on Caltech101/256}: adjusted mutual information scores given the correct number of clusters.}
\centering
\scalebox{.9}{
\begin{tabular}{cccc}
\toprule
     Methods &  Caltech101& Caltech256 & Privacy\\ 
     \midrule
     S10~\cite{Sculley2010a} & 0.733 & 0.630 &\xmark \\
     \midrule
     {\bf DPHE}&0.753&0.614 &\cmark \\
\bottomrule
\end{tabular}
}
\label{tab:clustering}
\end{table}

\subsection{Security Evaluation}
Finally, we introduce a formal version of our security evaluation on Algorithm 1 that supplements Section 2.5 in the original paper. Recall that our framework involves the following three types of parties:
\begin{definition}[Types of parties]
Let $U\ui{1},\dots,U\ui{N}$ be $N$ users, $A$ be an aggregator, and $G$ be a key generator. We assume that they are all semi-honest~\cite{OdedGoldreich2004} and do not collude. $U\ui{n}$ has private data $\wui\in\mathbb{R}^D$. $U\ui{n}$ can communicate only with $A$ and $G$, while $A$ and $G$ can communicate with all of the parties. $U\ui{j}\; (j\neq n)$ may be malicious and intercept data sent from $U\ui{n}$ to $A$.
\end{definition}

With DPHE, private data $\wui$ is first decomposed into $\wui=K\ui{n}\vct{v}\ui{n}$, where $\vct{v}\ui{n}\in\mathbb{R}^M$, $K\ui{n}\in\{0, 1\}^{D, M}$, and $M$ is an encryption capacity indicating the maximum number of values that are Paillier encrypted. $K\ui{n}$ is called an index matrix and defined as follows:
\begin{definition}[Index matrix]
An index matrix of $\vct{w}\in\mathbb{R}^D$ given $\vct{v}\in\mathbb{R}^M$ is a binary matrix $K\in\{0, 1\}^{D\times M}$ such that the number of ones for each column is exact one and that for each row is at most one and $\vct{w}=K\vct{v}$. Let $\mathcal{K}_{D, M}$ be a set of all possible index matrices of size $D\times M$. 
\label{def:ind}
\end{definition}
To make $\vct{v}\ui{n}$ secure, we use the Paillier encryption~\cite{Paillier1999a}. On the encrypted data $\he{\vct{v}\ui{n}}$, the following lemma holds based on Theorem 14 and Theorem 15 in \cite{Paillier1999a}:
\begin{lemma}[Paillier Encryption~\cite{Paillier1999a}]
Let $\he{\vct{v}}$ be a vector which is obtained by encrypting $\vct{v}$ with the Paillier cryptosystem. Then, no one can identify $\vct{v}$ from $\he{\vct{v}}$ without a decryption key.
\label{le:he}
\end{lemma}
Note that the Paillier encryption of $\vct{v}\ui{n}$ also helps to keep secret the number of non-zeros in $\wui$ since $M$ is always greater than or equal to the non-zero number.

On the other hand, DPHE doubly-permutes $K\ui{n}$, namely $\dperm{K\ui{n}}=\perm\ui{n}\perm K\ui{n}$, where $\perm\ui{n}, \perm \in\{0,1\}^{D\times D}$ is a permutation matrix defined as follows:
\begin{definition}[Permutation matrix]
A permutation matrix that permutes $D$ elements is a square binary matrix $\perm\in\{0, 1\}^{D\times D}$ such that the number of one for each column and for each row is exact one. Let $\varOmega_D$ be a set of all possible permutation matrices of size $D$. 
\label{def:perm}
\end{definition}
In DPHE, $U\ui{n}$ has both $\perm$ and $\perm\ui{n}$ but does not have $\{\perm\ui{j}\mid j\neq n\}$, while $A$ has $\{\perm\ui{n}\mid n=1,\dots,N\}$ but does not have $\perm$. In what follows we prove that without having both of $\perm$ and $\perm\ui{n}$, one cannot identify $K\ui{n}$ from $\dperm{K\ui{n}}$ (\ie, only $U\ui{n}$ can identify $K\ui{n}$). As preliminaries, we introduce several properties of index and permutation matrices based on Definition~\ref{def:ind} and Definition~\ref{def:perm}.
\begin{corollary}[Properties of index matrices]
For $\perm\in\varOmega_D$ and $K\in\mathcal{K}_{D, M}$, $\perm K \in \mathcal{K}_{D, M}$. 
\label{col:ind}
\end{corollary}
\begin{corollary}[Properties of permutation matrices]
For $\perm, \perm'\in \varOmega_D$, $\perm^{-1}=\perm^\top \in \varOmega_D$, and $\perm\perm'\in\varOmega_D$. 
\label{col:perm}
\end{corollary}
Then, the following lemma about $\dperm{K\ui{n}}$ holds:
\begin{lemma}[Reordering $\dperm{K\ui{n}}$]
Let $\dperm{K\ui{n}}=\perm\ui{n}\perm K\ui{n}\in\mathcal{K}_{D, M}$ where $\perm\ui{n}, \perm \in \varOmega_D$ and $K\ui{n}\in\mathcal{K}_{D, M}$. Suppose that $\dperm{K\ui{n}}$ is known and $K\ui{n}$ is unknown. Then, $K\ui{n}$ can be determined uniquely if and only if both $\perm$ and $\perm\ui{n}$ are known.
\label{le:getK}
\end{lemma}
\begin{proof}
If both of $\perm$ and $\perm\ui{n}$ are known, $K\ui{n}$ can be determined uniquely as $K\ui{n}=\perm^\top(\perm\ui{n})^\top \dperm{K\ui{n}}$ where the variables in the RHS are all known. To prove the `only-if' proposition, we introduce its contrapositive: `if at least one of $\perm$ and $\perm\ui{n}$ is unknown, $K\ui{n}$ cannot be determined uniquely.' Let $\perm'=\perm^\top(\perm\ui{n})^\top\in\varOmega_D$ (which is also a permutation matrix as shown in Corollary~\ref{col:perm}) which is unknown when at least one of $\perm$ and $\perm\ui{n}$ is unknown. For arbitrary $\perm'$, $K\ui{n}=\perm' \dperm{K\ui{n}}$ is always an index matrix as shown in Corollary~\ref{col:ind}. Therefore, $K\ui{n}$ cannot be determined uniquely as long as $\perm'$ is not fixed. This means that the contrapositive is true, and Lemma~\ref{le:getK} is proved. 
\end{proof}

The combination of Lemma~\ref{le:he} and Lemma~\ref{le:getK} proves that the aggregator $A$ and user $U\ui{j}\;(j\neq n)$ cannot identify $U\ui{n}$'s private data $\wui$ and its non-zero indices from encrypted data $\he{\vct{v}\ui{n}}, \dperm{K\ui{n}}$. The remaining concern is if one can identify $\wui$ or its non-zero indices from $\bar{\vct{w}} = \frac{1}{N}\sum_n \wui$. As shown in the original paper, when $N\geq 3$, it is impossible for any party to decompose $\sum_n \wui$ into individual $\wui$'s or to decompose non-zero indices of $\bar{\vct{w}}$ into those of individual $\wui$'s, as long as all parties are semi-honest and do not collude to share private information outside the algorithm.

To conclude, the following theorem is proved:
\begin{theorem}[Security on Algorithm 1]
After running Algorithm 1 by semi-honest and non-colluding parties $U\ui{1},\dots,U\ui{N}$, $A$, and $G$ where $N\geq 3$, no one but $U\ui{n}$ can identify private data $\wui$ and its non-zero indices from obtained information.
\end{theorem}

\paragraph{Acknowledgements}
This work was supported by JST ACT-I Grant Number JPMJPR16UT and JST CREST Grant Number JPMJCR14E1, Japan.

\balance
{\small
\bibliographystyle{ieee}
\bibliography{mendeley_copy}
}

\end{document}